\documentclass[twoside]{article}
\usepackage{amsmath, amssymb, amsthm}
\usepackage{mathtools}
\usepackage{microtype}
\usepackage{bm}
\usepackage{enumitem}
\usepackage[hidelinks]{hyperref}
\usepackage[capitalize,nameinlink]{cleveref}
\usepackage[preprint]{aistats2026}
\usepackage{natbib}
\newtheorem{theorem}{Theorem}
\newtheorem{lemma}{Lemma}
\newtheorem{definition}{Definition}
\newtheorem{corollary}{Corollary}
\theoremstyle{remark}
\usepackage{booktabs}     
\usepackage{multirow}    
\usepackage{array}
\newtheorem*{remark}{Remark}
\usepackage{framed,xcolor}
\usepackage{graphicx}
\definecolor{PCAccent}{RGB}{85,78,68}
\definecolor{PCBorder}{RGB}{227,222,214}
\definecolor{PCBack}{RGB}{252,250,246}
\definecolor{PCTitle}{RGB}{38,36,33}
\definecolor{PCBand}{RGB}{246,243,238}

\newcommand{\pccompactlist}{%
  \setlength{\itemsep}{2pt}%
  \setlength{\parsep}{0pt}%
  \setlength{\parskip}{0pt}%
  \setlength{\topsep}{2pt}%
}
\newenvironment{promptcard}[1]{%
  \par\medskip\noindent
  \def\FrameCommand##1{%
    {\color{PCAccent}\vrule width 1.4pt}\hspace{7pt}%
    \fboxsep=8pt\color{PCBorder}%
    \fcolorbox{PCBorder}{PCBack}{##1}%
  }%
  \MakeFramed{\advance\hsize-\width\FrameRestore}%
  {\setlength{\fboxsep}{4pt}%
   \colorbox{PCBand}{%
     \parbox{\dimexpr\linewidth-2\fboxsep\relax}{%
       {\color{PCTitle}\bfseries\scshape Prompt Card:\ }{\color{PCTitle}\bfseries #1}
     }%
   }%
  }\par\vspace{6pt}%
  \begingroup\small\color{black}%
  \setlength{\parskip}{3pt}\setlength{\parindent}{0pt}%
}{%
  \endgroup
  \endMakeFramed
  \par\medskip
}

\begin{document}

%
\runningtitle{Audit-of-Understanding: Posterior-Constrained Inference for Mathematical Reasoning}

%

\twocolumn[


\aistatstitle{Audit-of-Understanding: Posterior-Constrained Inference \\for Mathematical Reasoning in Language Models}

\aistatsauthor{ Samir Abdaljalil \And Erchin Serpedin \And  Khalid Qaraqe \And Hasan Kurban }

\aistatsaddress{ Texas A\&M University\\
  College Station, TX., USA\\ \{sabdaljalil, eserpedin\}@tamu.edu \And  Hamad Bin Khalifa University\\Doha, Qatar\\ 
  \{kqaraqe, hkurban\}@hbku.edu.qa}
]

\begin{abstract}
Large language models (LLMs) often generate reasoning traces that appear coherent but rest on unsupported assumptions, leading to hallucinated conclusions. Prior work mainly addresses factual hallucinations or relies on post-hoc verification, leaving reasoning-induced hallucinations largely unaddressed. We propose \emph{Audit-of-Understanding} (AoU), a framework that constrains inference to validated premises through three phases: (1) decomposing a query into candidate assumptions, (2) auditing their support, and (3) conditioning inference only on the validated subset. Formally, AoU is \emph{posterior-constrained inference}, connecting to selective prediction and rejection learning. Our contributions are threefold: (i) theoretical guarantees under perfect validation, (ii) excess-risk bounds under imperfect audits, and (iii) tractability analysis. Empirically, AoU improves both accuracy and faithfulness on GSM8K, MultiArith, and SVAMP, achieving up to +30\% gains on GSM8K, +45\% on MultiArith, and consistent +20--28\% improvements on SVAMP over Chain-of-Thought, Self-Consistency, and CoT-Decoding. Code is available at \url{https://anonymous.4open.science/r/audit-of-understanding-E28B}. 
\end{abstract}

\section{INTRODUCTION}
\label{sec:intro}

A central challenge in machine learning is ensuring that inference is \emph{faithful} to its underlying assumptions. 
Models often operate over latent structures—whether in probabilistic graphical models, structured prediction, or sequential reasoning—yet the assumptions driving inference are rarely surfaced or audited.  
As a result, predictions may depend on unsupported or speculative premises, yielding outputs that appear coherent but are not logically grounded.  This problem cuts across domains: from structured prediction in vision and language, to decision-making in reinforcement learning, to scientific and medical reasoning systems.  LLMs provide a vivid illustration of this issue. 
While capable of impressive multi-step reasoning, they frequently produce \emph{hallucinations}—outputs that rely on assumptions not given or entailed by the input 
\citep{maynez-etal-2020-faithfulness,pagnoni-etal-2021-understanding,ji-etal-2024-llm}. 

\begin{figure}
\begin{center}
\includegraphics[width=0.85\columnwidth]{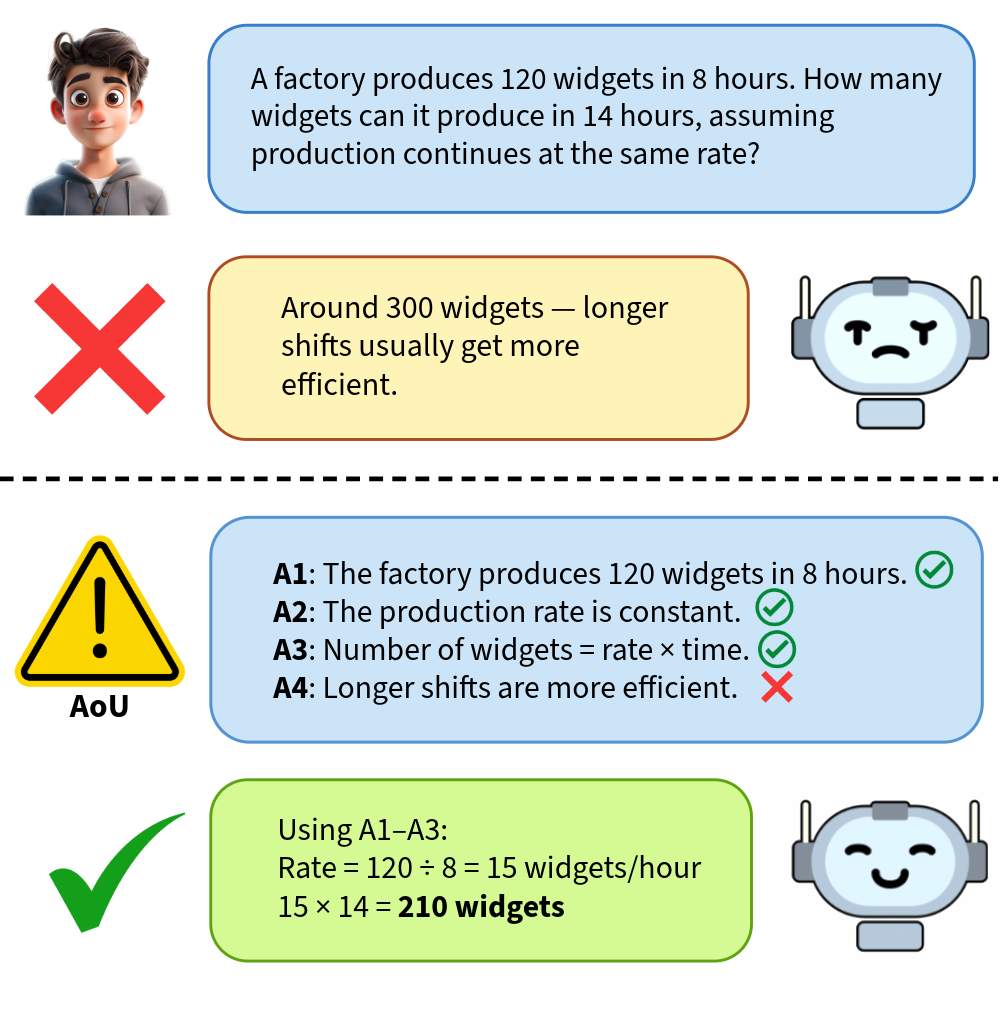}
\caption{Illustrative example of Audit-of-Understanding. The top shows reasoning with unchecked assumptions leading to an incorrect answer, while the bottom shows AoU filtering invalid assumptions, yielding the correct answer. }
\end{center}
 \label{fig:toy_example}
 \end{figure}
Prior work highlights that these hallucinations often emerge not from faulty recall, but from flawed intermediate reasoning steps 
\citep{turpin2023language,press-etal-2023-measuring}. 
Existing mitigation strategies—retrieval augmentation \citep{10.5555/3495724.3496517}, post-hoc verification \citep{gao2022,zheng2024critic}, 
and self-evaluation \citep{manakul-etal-2023-selfcheckgpt,dhuliawala2023chain,shinn2023reflexion}—primarily address factuality after inference has already occurred. 
Prompting methods such as Chain-of-Thought \citep{cot_wei} improve transparency but often introduce fabricated bridging facts, reinforcing the problem rather than eliminating it. 
What is missing is a principled approach that constrains inference to validated assumptions \emph{before} prediction.

We introduce Audit-of-Understanding (AoU), a general framework for faithful inference under partial support. 
AoU decomposes a query into candidate premises, validates which are supported, and conditions inference only on the validated subset. 
Fig. \ref{fig:toy_example} presents an example that illustrates how AoU works. 
Formally, AoU can be cast as a constrained inference problem (see Section~\ref{sec:method}) \citep{cortes2016learning,chow2009optimum} as well as constrained inference in probabilistic graphical models \citep{koller2009probabilistic}. 
Eliminating unsupported assumptions prior to prediction yields formal faithfulness and risk-control guarantees (see Section~\ref{sec:method}), while producing interpretable reasoning traces.

Our contributions are threefold:  
(1) We formalize AoU as \emph{posterior-constrained inference}, with definitions, guarantees, and complexity analysis;  
(2) We derive excess risk bounds under imperfect validation, linking validator reliability to prediction risk;  
and (3) We demonstrate empirically, on mathematical reasoning benchmarks, that AoU substantially reduces hallucinations while preserving accuracy. 
Although our experiments focus on LLM reasoning, the framework applies broadly to any learning setting where inference depends on unvalidated intermediate assumptions.


\section{BACKGROUND AND RELATED WORK}
\label{sec:related}

\subsection{Faithful Inference and Selective Classification}
Ensuring that predictions are faithful to validated information is a longstanding problem in machine learning. 
Selective classification and rejection-option learning \citep{chow2009optimum,cortes2016learning} formalize settings 
where a predictor abstains rather than risk making unsupported inferences. 
Similarly, probabilistic graphical models study inference under structural or evidence constraints, 
where efficiency depends on bounded treewidth and principled marginalization strategies \citep{koller2009probabilistic}. 
Our work extends these ideas by proposing a framework in which inference is explicitly conditioned on validated premises, 
providing guarantees against unsupported reasoning traces.

\subsection{Hallucination in Generative Models}
Generative models, and particularly LLMs, make the challenge of faithfulness concrete. 
Hallucination—defined as producing fluent but factually unsupported content—has been studied extensively in summarization \citep{maynez-etal-2020-faithfulness,pagnoni-etal-2021-understanding}, 
question answering \citep{manakul-etal-2023-selfcheckgpt}, 
and instruction following \citep{ji-etal-2024-llm}. 
Most approaches aim to reduce factual errors through retrieval-augmented generation \citep{10.5555/3495724.3496517}, 
post-hoc verification \citep{gao2022,zheng2024critic,gou2024critic}, 
or self-consistency checks \citep{dhuliawala2023chain,shinn2023reflexion}. 
While effective at improving factual alignment, these methods largely address hallucinations after the reasoning process has already occurred.

\subsection{Reasoning-Induced Hallucination}
Recent work highlights that hallucinations can also arise from flawed intermediate reasoning. 
Chain-of-Thought prompting \citep{cot_wei} improves performance but often introduces fabricated intermediate steps 
\citep{turpin2023language}, while methods such as Self-Ask \citep{press-etal-2023-measuring} and ReAct 
\citep{DBLP:journals/corr/abs-2210-03629} scaffold reasoning into subproblems without distinguishing 
between grounded and speculative assumptions. 
Post-hoc methods like Chain-of-Verification \citep{dhuliawala2023chain} or Critic-Prompting \citep{zheng2024critic,gou2024critic} 
add verification after inference, but unsupported assumptions may already have influenced predictions. 
Our work differs by introducing an explicit \emph{audit phase}, which separates given, inferred, and missing premises before inference begins.


\subsection{Mathematical Reasoning}
Large language models have recently demonstrated surprising competence in mathematical problem solving, but their reasoning remains inconsistent and prone to systematic errors. Early studies (e.g., \citet{cobbe2021trainingverifierssolvemath}) showed that transformer-based models can solve arithmetic and algebraic problems when trained on synthetic data, but their performance sharply drops on out-of-distribution tasks. Chain-of-thought prompting \citep{cot_wei} and self-consistency decoding \citep{wang2023selfconsistency} improve accuracy by encouraging step-by-step derivations, yet they also introduce hallucinated intermediate steps and unstable final answers, particularly on multi-step or compositional tasks. Our AoU framework complements this line of research by introducing an explicit pre-generation audit of assumptions and intermediate claims, with the goal of reducing unsupported steps and increasing faithfulness in mathematical reasoning without relying on external tools.

\begin{figure*}[ht]
\centering
\includegraphics[width=0.9\textwidth]{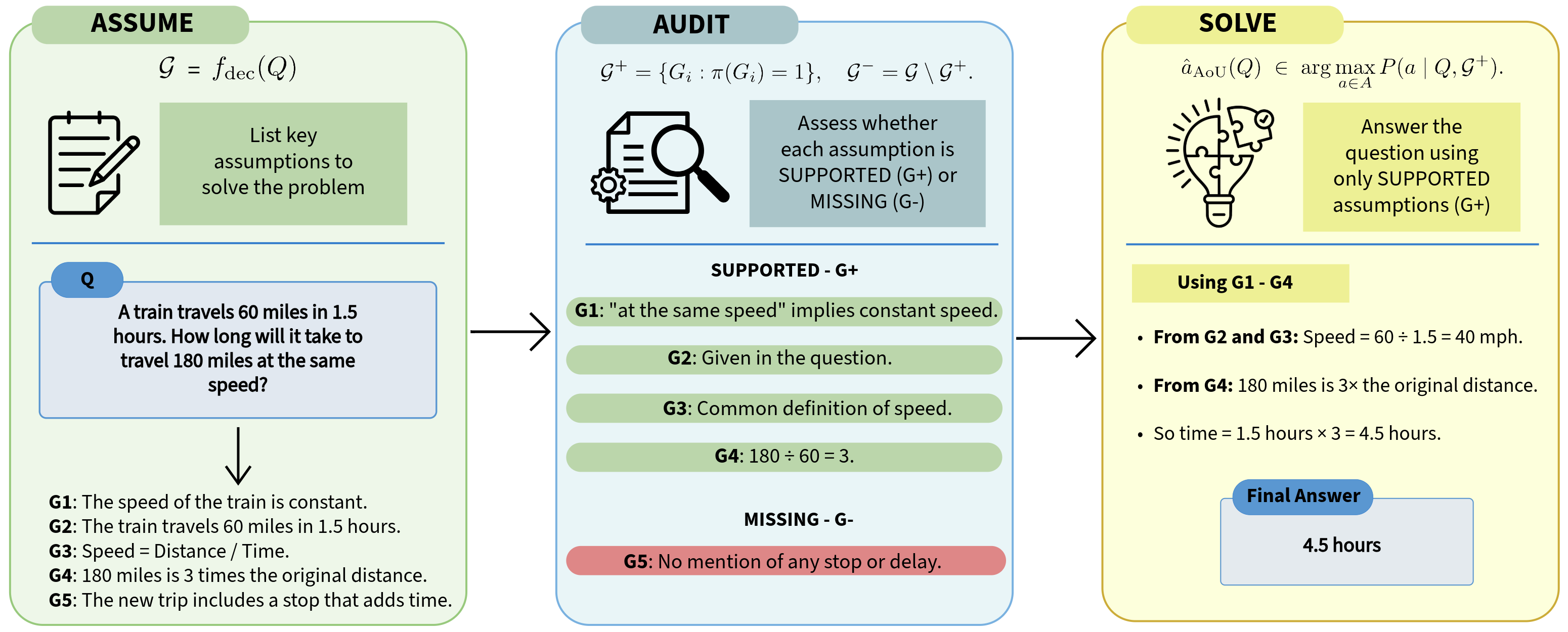} 
\caption{Illustration of the Audit-of-Understanding (AoU) pipeline on a real math word problem.
The pipeline consists of three phases: (1) Assume,
(2) Audit,
and (3) Solve.
This prevents hallucinated reasoning steps and ensures faithfulness. The model uses G1–G4 to compute the speed and deduce that the 180-mile trip will take 4.5 hours.}
\label{fig:method}
\end{figure*}

\section{Methodology}
\label{sec:method}

We formalize AoU as a posterior-constrained inference framework over latent assumptions. 
Let $Q$ denote the input query, $\mathcal A$ the finite action space, and 
$\mathcal{G} = \{G_1, \dots, G_m\}$ the set of candidate assumptions generated by the model.  Whereas standard prompting marginalizes over all of $\mathcal{G}$, AoU restricts reasoning to a validated subset $\mathcal{G}^+ \subseteq \mathcal{G}$. The pipeline is illustrated in Fig.~\ref{fig:method}.

\paragraph{Notation.}
For a set of variables $\mathcal H$, let $\mathcal{S}(\mathcal H)$ denote their joint configuration space, and for a single variable $G_i$, let $\mathcal{S}(G_i)$ denote its state space. 
Bold symbols represent tuples, e.g., $\boldsymbol{g} \in \mathcal{S}(\mathcal{G})$.  The ground-truth answer for query $Q$ is denoted $A^\star \in \mathcal Y$ (we suppress the explicit
dependence on $Q$ when unambiguous), and losses are measured by a bounded function
$L: \mathcal A \times \mathcal Y \to [0,1]$. We denote by $\hat a(Q)$ an arbitrary predictor, by $\hat a_{\mathrm{AoU}}(Q)$ the AoU predictor under perfect validation, by $\hat a_\pi(Q)$ the AoU predictor with a (possibly imperfect) validator $\pi$, and by $a^\dagger(Q)$ the Bayes action under perfect validation.  Let $\mathcal{Q}$ denote the space of possible queries $Q$.\\
\noindent\emph{Semantics.} Each $G_i$ is a proposition encoded as a discrete random variable
($G_i\in\mathcal S(G_i)$); support means $G_i\in S_i^{+}\subseteq\mathcal S(G_i)$.
Under \emph{perfect validation}, $S_i^{+}$ is the true set of supported states (no uncertainty).
\paragraph{Standing assumptions.}
We assume the action set $\mathcal A$ is finite (ties broken by a fixed rule) and use a bounded loss $L\in[0,1]$.
\noindent\paragraph{Cost model.}
Unless stated otherwise, we adopt a unit-cost oracle model in which each validator call and each
evaluation of a label-conditional probability $P(y\mid\cdot)$ costs $O(1)$. In concrete LLM settings,
token-level computation implies non-constant costs; with per-call constants $C_\pi$ and $C_P$, complexities
become $O(m\,C_\pi)$ for validation and the stated inference terms multiplied by $C_P$.

\begin{promptcard}{1 - Assumption Enumeration}
\textbf{Instruction.} Given a task or question, enumerate the minimal set of assumptions, facts, or subgoals required to reach a solution. Do not solve the task.

\medskip
\textbf{Constraints.}
\begin{itemize}\pccompactlist
  \item List only essential items, labeled G1, G2, G3, $\ldots$
  \item Avoid trivial, vague, or redundant statements
  \item Each entry must be precise, concise, and problem-specific
\end{itemize}

\medskip
\textbf{Input.} \texttt{Question: \{q\}} \quad
\textbf{Output.} \texttt{G1, G2, G3, \dots}
\end{promptcard}

\subsection{Phase 1: Decomposition of Reasoning Requirements}

Given an input $Q$, the model produces a set of required premises
\begin{equation}
\mathcal{G} = f_{\text{dec}}(Q),
\label{eq:decomposition}
\end{equation}
where $f_{\text{dec}}$ is a decomposition function. 
Each $G_i \in \mathcal{G}$ is labeled as
\begin{equation}
\ell_i \in \{\texttt{GIVEN}, \texttt{INFERRED}, \texttt{MISSING}\}.
\label{eq:labels}
\end{equation}
This step surfaces the model’s internal assumptions before inference. 
In Prompt Card 1, we show the prompt designed to enumerate all assumptions required to answer the question.

\subsection{Phase 2: Assumption Validation}

A validator $\pi: \mathcal{G} \to \{0,1\}$ audits each premise:
\begin{equation}
\pi(G_i) =
\begin{cases}
1 & \text{if $G_i$ is supported or logically entailed}, \\
0 & \text{otherwise}.
\end{cases}
\label{eq:validator}
\end{equation}
We define the validated and rejected subsets:
\begin{equation}
\mathcal{G}^+ = \{ G_i : \pi(G_i)=1\}, 
\quad
\mathcal{G}^- = \mathcal{G}\setminus \mathcal{G}^+.
\label{eq:gplusminus}
\end{equation}
\noindent
Unless stated otherwise, $\pi$ is assumed deterministic. Section~\ref{sec:imperfect} treats probabilistic
validators with per-premise false-positive/false-negative rates, modeling auditing errors.
\noindent\textit{Note.} The validated set $\mathcal G^{+}$ is determined solely by $\pi$, regardless of the
Phase~1 tags $\ell_i$; i.e., $\pi$ overrides $\ell_i$ when forming $\mathcal G^{+}$. Consequently,
all faithfulness guarantees below are independent of Phase~1 tagging accuracy. In Prompt Card 2, we showcase the prompt used to audit the assumptions.

\begin{promptcard}{2 - Audit Phase}
\textbf{Instruction.} Given (i) a question and (ii) a set of assumptions $(G_1, G_2, \ldots)$, assess whether each assumption is supported by the question or by unambiguous implications thereof. Do not introduce external knowledge. 

\medskip
\textbf{Evaluation Rules.}
\begin{itemize}\pccompactlist
  \item For each $G_i$, assign a label: \texttt{[SUPPORTED]} or \texttt{[MISSING]}.
  \item Provide a brief justification for every label; be strict and conservative.
  \item Do not invent facts; rely only on what is explicitly stated or clearly inferable.
\end{itemize}

\medskip
\textbf{Input.}\\
\texttt{Question: \{q\}}\\
\texttt{Assumptions: \{G1, G2, \dots\}}

\medskip
\textbf{Output.} One line per assumption, e.g.:\\
\texttt{G1 — [SUPPORTED]: <short reason>}\\
\texttt{G2 — [MISSING]: <short reason>}\\
\texttt{G3 — [SUPPORTED]: <short reason>}
\end{promptcard}

\subsection{Phase 3: Constrained Inference}

AoU conditions inference on $\mathcal{G}^+$ and defines the AoU predictor
\begin{equation}
\hat a_{\mathrm{AoU}}(Q)\;\in\;\arg\min_{a\in\mathcal A}\ \mathbb{E}\!\left[L\!\big(a,A^\star\big)\ \middle|\ Q,\mathcal G^{+}\right].
\label{eq:constrained_decision}
\end{equation}

\noindent\textit{Special case (0--1 loss).}
When $\mathcal A=\mathcal Y$ and $L$ is $0$--$1$, \eqref{eq:constrained_decision} reduces to
$\hat a_{\mathrm{AoU}}(Q)\in\arg\max_{y\in\mathcal Y}P(y\mid Q,\mathcal G^{+})$.\\

\noindent\textit{Shorthand.} We write $P(y \mid Q,\mathcal G^{+})$ for the predictive distribution over labels conditioned on the conjunction of validation events, i.e.,
$P\!\left(y \,\middle|\, Q, \bigcap_{i:\,G_i\in\mathcal G^{+}} \{G_i \in S_i^{+}\}\right)$. (Throughout this section we write $P(y\mid\cdot)$ for brevity when referring to the predictive distribution over labels.)
When essential information is missing, we marginalize over $\mathcal{G}^-$ conditioned on $\mathcal{G}^+$:
\begin{equation}
P(y \mid Q, \mathcal{G}^+) \;=\; 
\sum_{\boldsymbol{g}\in \mathcal{S}(\mathcal{G}^-)}
P\!\left(y \mid Q, \mathcal{G}^+, \boldsymbol{g}\right)\,
P\!\left(\boldsymbol{g}\mid Q,\mathcal{G}^+\right).
\label{eq:conditional_reasoning}
\end{equation}
Equivalently, the decision risk can be written as
\[
\mathbb{E}\!\left[L(a,A^\star)\mid Q,\mathcal G^{+}\right]
=\sum_{y\in\mathcal Y} L(a,y)\,P(y\mid Q,\mathcal G^{+}).
\]

\begin{remark}[Hypothetical completions]
The distribution $P(\boldsymbol{g} \mid Q,\mathcal{G}^+)$ is a device for
marginalizing over \emph{unvalidated} premises; it need not represent grounded evidence.
In applications, one may choose conservative families (e.g., low-entropy or adversarial
priors) or trigger abstention when the posterior mass on $\mathcal{G}^-$ is large.
\end{remark}
In Prompt Card 3, we provide the prompt used to elicit the final answer through constrained inference, relying only on the supported assumptions.

\subsection{Bayesian Interpretation}

Standard prompting computes
\begin{equation}
P(y \mid Q) = \sum_{\boldsymbol{g}\in\mathcal{S}(\mathcal{G})}
P(y \mid Q, \boldsymbol{g})\, P(\boldsymbol{g}\mid Q).
\label{eq:standard_inference}
\end{equation} 

Standard prompting marginalizes over all premises (Eq.~\ref{eq:standard_inference}); AoU replaces this with the constrained decision in Eq.~\ref{eq:constrained_decision}.

\subsection{Faithfulness and Complexity Guarantees}

\begin{definition}[Reasoning trace and minimal dependence set]
For $S \subseteq \mathcal{G}$, write $\boldsymbol{g} \equiv_S \boldsymbol{g}'$ if 
$g_i = g'_i$ for all $G_i \in S$. Let $K_{\hat a}(a \mid Q,\boldsymbol{g})$ denote 
the (possibly randomized) conditional kernel used at inference time, where 
$\boldsymbol{g} \in \mathcal{S}(\mathcal{G})$ assigns values to all premises. In generative settings, $K_{\hat a}$ may coincide with sampling or decoding procedures derived from the posterior $P(\cdot)$; we keep them notationally distinct to separate model posteriors from the decision kernel.

A set $S \subseteq \mathcal{G}$ is a \emph{support set} for $\hat a(Q)$ if
\[
K_{\hat a}(a \mid Q,\boldsymbol{g})
\;=\;
K_{\hat a}(a \mid Q,\boldsymbol{g}')
\quad \forall a \in \mathcal A
\]
whenever $\boldsymbol{g} \equiv_S \boldsymbol{g}'$.
Since $\mathcal{G}$ is finite, inclusion-minimal support sets exist but need not be unique. 
We define the \emph{reasoning trace} $\mathcal T(\hat a(Q))$ as the \emph{intersection of all} 
inclusion-minimal support sets (the indispensable premises). 
Its length is $k_{\text{trace}} = |\mathcal T(\hat a(Q))|$.
\end{definition}
\begin{promptcard}{3 - Constrained Reasoning}
\textbf{Instruction.} Given a question and a set of \emph{audited} assumptions $(G_1, G_2, \ldots)$ with labels \texttt{[SUPPORTED]} or \texttt{[MISSING]}, answer the question using only the \texttt{[SUPPORTED]} assumptions.

\medskip
\textbf{Rules.}
\begin{itemize}\pccompactlist
  \item Reference assumptions by index (e.g., “Using G2, …”).
  \item If any critical assumptions are \texttt{[MISSING]}, provide a conditional answer or state why an exact answer is not possible.
  \item Keep the explanation clear, concise, and strictly grounded in the cited assumptions.
\end{itemize}

\medskip
\textbf{Input.}\\
\texttt{Question: \{q\}}\\
\texttt{Audited Assumptions: \{G1: [SUPPORTED] <reason>, G2: [MISSING] <reason>, \dots\}}

\medskip
\textbf{Output.} A brief answer grounded in \texttt{[SUPPORTED]} assumptions, with inline citations to $G_i$. For example:\\
\texttt{Answer: <2--4 sentences referencing G1, G3, ...>}\\
\texttt{If missing: “Under G1 and G3, the result is X; however, without G2 [MISSING], the exact value of Y cannot be determined.”}
\end{promptcard}

\begin{theorem}[Trace faithfulness without unsupported premises]
Under perfect validation, the reasoning trace of $\hat a_{\mathrm{AoU}}$ satisfies
\[
\mathcal T(\hat a_{\mathrm{AoU}}(Q)) \subseteq \mathcal G^{+}
\quad\text{and}\quad
\mathcal T(\hat a_{\mathrm{AoU}}(Q)) \cap \mathcal G^{-}=\varnothing .
\]
\end{theorem}

\begin{proof}
By construction, $\hat a_{\mathrm{AoU}}$ depends only on $(Q,\mathcal G^{+})$.
Thus for any $\boldsymbol g,\boldsymbol g'$ agreeing on $\mathcal G^{+}$,
$K_{\hat a_{\mathrm{AoU}}}(\cdot\mid Q,\boldsymbol g)=
K_{\hat a_{\mathrm{AoU}}}(\cdot\mid Q,\boldsymbol g')$, so the decision is invariant
to assignments on $\mathcal G^{-}$. Therefore
$\mathcal T(\hat a_{\mathrm{AoU}}(Q))\subseteq \mathcal G^{+}$ and
$\mathcal T(\hat a_{\mathrm{AoU}}(Q))\cap \mathcal G^{-}=\varnothing$.
\end{proof}
The decision may depend on realized values within $\mathcal G^{+}$, so the trace is a
subset of $\mathcal G^{+}$ in general.

\begin{definition}[Hallucination]
For a predictor $\hat a$, define the hallucination event
\[
H(\hat a) := \big[\mathcal{T}(\hat a(Q)) \cap \mathcal{G}^- \neq \varnothing\big].
\]
\end{definition}

\begin{corollary}[No hallucination under perfect validation]
With perfect validation,
\[
\mathbb P\!\big[\,H(\hat a_{\mathrm{AoU}})\,\big] \;=\; 0.
\]
\end{corollary}

\begin{proof}
Under perfect validation, $\mathcal{G}^+$ contains only supported/entailed items and
$\mathcal{G}^-$ only unsupported ones. Since $\hat a_{\mathrm{AoU}}$ conditions solely on
$\mathcal{G}^+$, its reasoning trace cannot use any premise in $\mathcal{G}^-$. Hence
$H(\hat a_{\mathrm{AoU}})$ is impossible, so its probability is $0$.
\end{proof}

We write $u(\mathcal G^{+})$ for the number of constant-time statistic updates required to
incorporate $\mathcal G^{+}$ into the factorization used to evaluate $P(y\mid\cdot)$ (e.g., the predictive distribution);
in a log-linear model each validated premise triggers an $O(1)$ feature update.

\begin{theorem}[Complexity Bound]
\label{thm:complexity}
Let $m=|\mathcal{G}|$ and $k=|\mathcal{G}^-|$. Then:
\begin{enumerate}[leftmargin=*]
    \item Validation requires $O(m)$ calls to $\pi$.  
\item \textbf{Decision cost.} Under a scoring oracle that evaluates $\mathbb{E}[L(a,A^\star)\mid Q,\mathcal G^{+}]$ in $O(1)$ per $a$ (e.g., via $P(y\mid Q,\mathcal G^{+})$ and a finite $\mathcal Y$), computing \eqref{eq:constrained_decision} costs $O(|\mathcal A|)$. If the predictive distribution factorizes so that each validated premise updates sufficient statistics in $O(1)$, the additional incorporation cost is $u(\mathcal G^{+})$, yielding $O(|\mathcal A|+u(\mathcal G^{+}))$ overall (times a per-call constant $C_P$ in autoregressive settings).
    \item  Let $k = |\mathcal{G}^-|$ and let $n$ denote the number of variables in the factorization
used for marginalization over $\mathcal{G}^-$. Worst-case complexity is $O(2^k)$ for binary premises,
and more generally $O\!\big(\prod_{i=1}^k |\mathcal{S}(G_i)|\big)$. Under bounded treewidth $t$ and bounded domain sizes,
junction-tree inference runs in $O(n \exp(t{+}1))$, polynomial in $n$ for fixed $t$
\citep{koller2009probabilistic}. More generally, complexity scales with the product of
domain sizes over maximal cliques. If independence holds across $G_i$, the cost reduces to
$O\!\left(\sum_{i=1}^k |\mathcal{S}(G_i)|\right)$.
\end{enumerate}
\end{theorem}

\begin{proof}
Validation inspects each $G_i$, giving $O(m)$ calls. 
For Item 2, computing $\hat a_{\mathrm{AoU}}(Q)=\arg\min_{a\in \mathcal A} \mathbb{E}[L(a,A^\star)\mid Q,\mathcal{G}^+]$ by a naïve scan over $\mathcal A$ uses $O(|\mathcal A|)$ expected-loss evaluations (e.g., via $P(y\mid Q,\mathcal G^{+})$ and finite $\mathcal Y$). If the predictive distribution factorizes with $O(1)$-time updates per validated premise, the total becomes $O(|\mathcal A|+u(\mathcal{G}^+))$ (as stated), times a per-call constant $C_P$ in autoregressive settings.
For Item 3, marginalizing over $k=|\mathcal{G}^-|$ binary premises is $O(2^k)$; more generally $O\!\big(\prod_{i=1}^k |\mathcal{S}(G_i)|\big)$. Under bounded treewidth $t$, junction-tree inference costs $O(n\exp(t{+}1))$ over $n$ variables in the $\mathcal{G}^-$ subgraph.
\end{proof}

\subsection{Error Decomposition and Imperfect Validation}
\label{sec:imperfect}

\paragraph{Validator error rates.}
For each premise $G_i$, let
$\alpha_i := \Pr[\pi(G_i)=0 \mid G_i \text{ supported}]$ (false negative)
and
$\beta_i := \Pr[\pi(G_i)=1 \mid G_i \text{ unsupported}]$ (false positive).

This analysis parallels selective classification and rejection-option learning 
\citep{cortes2016learning,chow2009optimum}, where the reliability of the 
validator directly bounds the excess risk of the predictor.

\begin{definition}[Indicator Function]
For any event $E$,
\[
\mathbf{1}\{E\} =
\begin{cases}
1 & \text{if $E$ holds}, \\
0 & \text{otherwise}.
\end{cases}
\]
\end{definition}

\paragraph{Probability space.}
All probabilities and expectations are taken with respect to the joint measure over
queries $Q\sim \mathsf P_Q$, the ground-truth label $A^\star$ given $Q$, any internal
randomization of the inference kernel $K_{\hat a}$, and (when applicable) the validator
$\pi$. Unless stated otherwise, per-premise validator error rates $(\alpha_i,\beta_i)$
are assumed query-averaged and uniform; the non-uniform case follows by replacing
$\alpha_i,\beta_i$ with $\mathbb E_Q[\alpha_i(Q)]$ and $\mathbb E_Q[\beta_i(Q)]$.

\paragraph{Risk.}
We define the total risk of a predictor $\hat a$ under the joint probability space
specified above as
\[
\mathcal R(\hat a)\;:=\;\mathbb{E}\big[L(\hat a(Q),A^\star)\big],
\]
where the expectation is taken over queries $Q\sim\mathsf P_Q$, the ground-truth
$A^\star$ given $Q$, and any internal randomness of the inference kernel $K_{\hat a}$
and, when applicable, the validator $\pi$.
For AoU with abstention cost $\lambda\in[0,1]$, we use the modified loss
\[
L_\lambda(a,A^\star)=
\begin{cases}
L(a,A^\star) & a\in\mathcal A,\\
\lambda & a=\bot,
\end{cases}
\]
and write $\mathcal R_\lambda(\hat a):=\mathbb E[L_\lambda(\hat a(Q),A^\star)]$.

\begin{lemma}[Error Decomposition]
\label{lem:decomp}
Let $H(\hat a)$ denote the hallucination event. Then the total risk $\mathcal R(\hat a)$ decomposes as
\begin{equation}
\mathcal R 
= \underbrace{\mathbb{E}\!\left[L(\hat a, A^\star)\,\mathbf{1}\{H(\hat a)\}\right]}_{\text{assumption error}}
+ \underbrace{\mathbb{E}\!\left[L(\hat a, A^\star)\,\mathbf{1}\{\neg H(\hat a)\}\right]}_{\text{inference error}}.
\end{equation}
Equivalently,
\begin{equation}
\begin{split}
\mathcal R 
&= \underbrace{\mathbb{E}\!\left[L(\hat a, A^\star)\mid H(\hat a)\right]
   \mathbb P(H(\hat a))}_{\text{assumption error}} \\
&\quad + \underbrace{\mathbb{E}\!\left[L(\hat a, A^\star)\mid \neg H(\hat a)\right]
   \mathbb P(\neg H(\hat a))}_{\text{inference error}}.
\end{split}
\end{equation}
\end{lemma}

\begin{proof}
Since $\mathbf{1}\{H(\hat a)\}+\mathbf{1}\{\neg H(\hat a)\}=1$, we have
\[
L(\hat a,A^\star) =
L(\hat a,A^\star)\,\mathbf{1}\{H(\hat a)\}
+ L(\hat a,A^\star)\,\mathbf{1}\{\neg H(\hat a)\}.
\]
Taking expectations yields the first decomposition.  
For the second form, apply the law of total expectation:
\[
\mathbb{E}[L(\hat a,A^\star)\,\mathbf{1}\{H(\hat a)\}]
= \mathbb{E}[L(\hat a,A^\star)\mid H(\hat a)]\mathbb P(H(\hat a)),
\]
and similarly for $\neg H(\hat a)$. This completes the proof.
\end{proof}

\paragraph{Bayes action under AoU.}
Let $a^\dagger(Q)\in\arg\min_{a\in\mathcal A}\mathbb E[L(a,A^\star)\mid Q,\mathcal G^{+}]$.
A premise is “used” when $G_i\in\mathcal T(\cdot)$. Let $\Delta_i\in[0,1]$ upper bound the
marginal loss increase from withholding such an indispensable premise.

\begin{theorem}[Excess-risk upper bound under probabilistic validation]
\label{thm:imperfect-strong}
Using the bounded loss from Section~\ref{sec:method}, let $a^\dagger(Q)\in\arg\min_{a} \mathbb E[L(a,A^\star)\mid Q,\mathcal G^+]$ be the Bayes action under perfect validation. Let $\hat a_\pi$ be AoU with a validator
having per-premise rates $(\alpha_i,\beta_i)$. Define $U_i$ (unsupported and used in
$\mathcal T(\hat a_\pi(Q))$) and $S_i$ (supported, would be used by $a^\dagger$, but
excluded due to FN). Let $p_i^{\text{use}}:=\mathbb E_Q[\mathbf 1\{G_i\in\mathcal T(a^\dagger(Q))\}]$
and $\Delta_i\in[0,1]$ be the maximal marginal loss increase when withholding $G_i$.

Then
\begin{equation*}
\begin{split}
\mathcal R(\hat a_\pi)-\mathcal R(a^\dagger)
\ \le\ 
& \underbrace{\sum_{i=1}^m \beta_i\,\mathbb P(G_i\ \mathrm{unsupported})\,p_i^{\text{use}\mid\mathrm{unsup}}}_{\text{FP inclusion}} \\
&\quad +\ 
\underbrace{\sum_{i=1}^m p_i^{\text{use}}\,\alpha_i\,\Delta_i}_{\text{FN exclusion}} .
\end{split}
\end{equation*}
where $p_i^{\text{use}\mid\mathrm{unsup}}$ is the probability that $G_i$ enters the trace given it slipped in unsupported.
\end{theorem}

The FN term is a union-bound–style decomposition: it sums marginal loss increments  and can over-count when multiple premises are jointly excluded. Thus the bound may be loose, but it is always safe. A refinement conditions on trace usage under unsupported premises (for the FP part) 
and on counterfactual trace usage when supported premises are withheld (for the FN part):
$\mathbb P[U_i] \le \beta_i\,\mathbb P(G_i\text{ unsupported})\,p_i^{\text{use}\mid \mathrm{unsup}}$,
at the cost of introducing $p_i^{\text{use}\mid \mathrm{unsup}}$.

\begin{proof}
By decomposition, validation errors affect risk in two ways: (i) false positives (FP) allow unsupported premises into $\mathcal{G}^+$, creating events $U_i$, and (ii) false negatives (FN) remove supported premises that would otherwise be useful, creating events $S_i$.  
Since $L\le1$, the contribution of FP errors is at most the expected number of unsupported premises used, $\sum_i \mathbb P[U_i]$. FN errors contribute at most $\sum_i p_i^{\text{use}}\,\alpha_i\,\Delta_i$, where $\Delta_i$ bounds marginal loss increase from excluding $G_i$. This yields the stated decomposition. The upper bound on $\sum_i \mathbb P[U_i]$ follows from $\mathbb P[\pi(G_i)=1 \mid G_i \text{ unsupported}] = \beta_i$. 
\end{proof}

\begin{corollary}[Tightened homogeneous-rate bound]
\label{cor:tight}
Under homogeneous rates $\alpha_i\equiv\alpha$, $\beta_i\equiv\beta$ and the bounded loss of Section~\ref{sec:method},
\[
\mathcal R(\hat a_\pi) \ \le\
\mathcal R(a^\dagger)
\ +\ \mathbb{E}\!\left[\min\{1,K_{\text{FP}}\}\right]
\ +\ \alpha\,\Delta_{\max}\,\mathbb{E}[k_{\text{trace}}^\dagger],
\]
where $k_{\text{trace}}^\dagger := |\mathcal{T}(a^\dagger(Q))|$ is the trace length under perfect validation,
$\Delta_{\max} := \max_{i\in\{1,\dots,m\}} \Delta_i$, and
$K_{\text{FP}}$ is the (random) number of false-positive premises actually used in the reasoning trace.
\end{corollary}

\paragraph{Abstention.}
\begin{definition}[Abstention Action]
We extend the action space to $\mathcal A \cup \{\bot\}$, where $\bot$ denotes the abstention (or rejection) option: the model outputs no answer.
\end{definition}
With abstention loss
\[
L_\lambda(a, A^\star)=
\begin{cases}
L(a,A^\star) & a\in \mathcal A,\\
\lambda & a=\bot,
\end{cases}
\]
where $\lambda\in[0,1]$, all results extend. 
Since the loss is bounded, choosing $\lambda$ below the expected loss of uninformed guessing ensures a strict risk reduction, consistent with Chow’s rejection rule \citep{chow2009optimum}.

\paragraph{Decision rule with abstention.}
For general loss $L$ and abstention cost $\lambda\in[0,1]$, abstain whenever
\[
\min_{a\in\mathcal A}\ \mathbb E[L(a,A^\star)\mid Q,\mathcal G^{+}]\ \ge\ \lambda,
\]
and otherwise output $\arg\min_{a\in\mathcal A}\mathbb E[L(a,A^\star)\mid Q,\mathcal G^{+}]$.
In the 0--1 case with $\mathcal A=\mathcal Y$, this reduces to abstaining whenever $\max_{y\in\mathcal Y}P(y\mid Q,\mathcal G^{+})\le 1-\lambda$.

\section{Experiments}
\label{sec:exp}

\paragraph{Datasets.} 
Our approach was evaluated on three mathematical reasoning benchmarks. \textsc{MultiArith}~\citep{roy-roth-2015-solving} consists of multi-step arithmetic word problems requiring compositional numerical reasoning, \textsc{GSM8K}~\citep{cobbe2021trainingverifierssolvemath}, which contains high school mathematical word problems, and \textsc{SVAMP}~\citep{patel-etal-2021-nlp}, a dataset of mathematical problems of fourth grade level, which require no more than two-hop reasoning.

\begin{figure*}[ht]
\centering
\includegraphics[width=\textwidth]{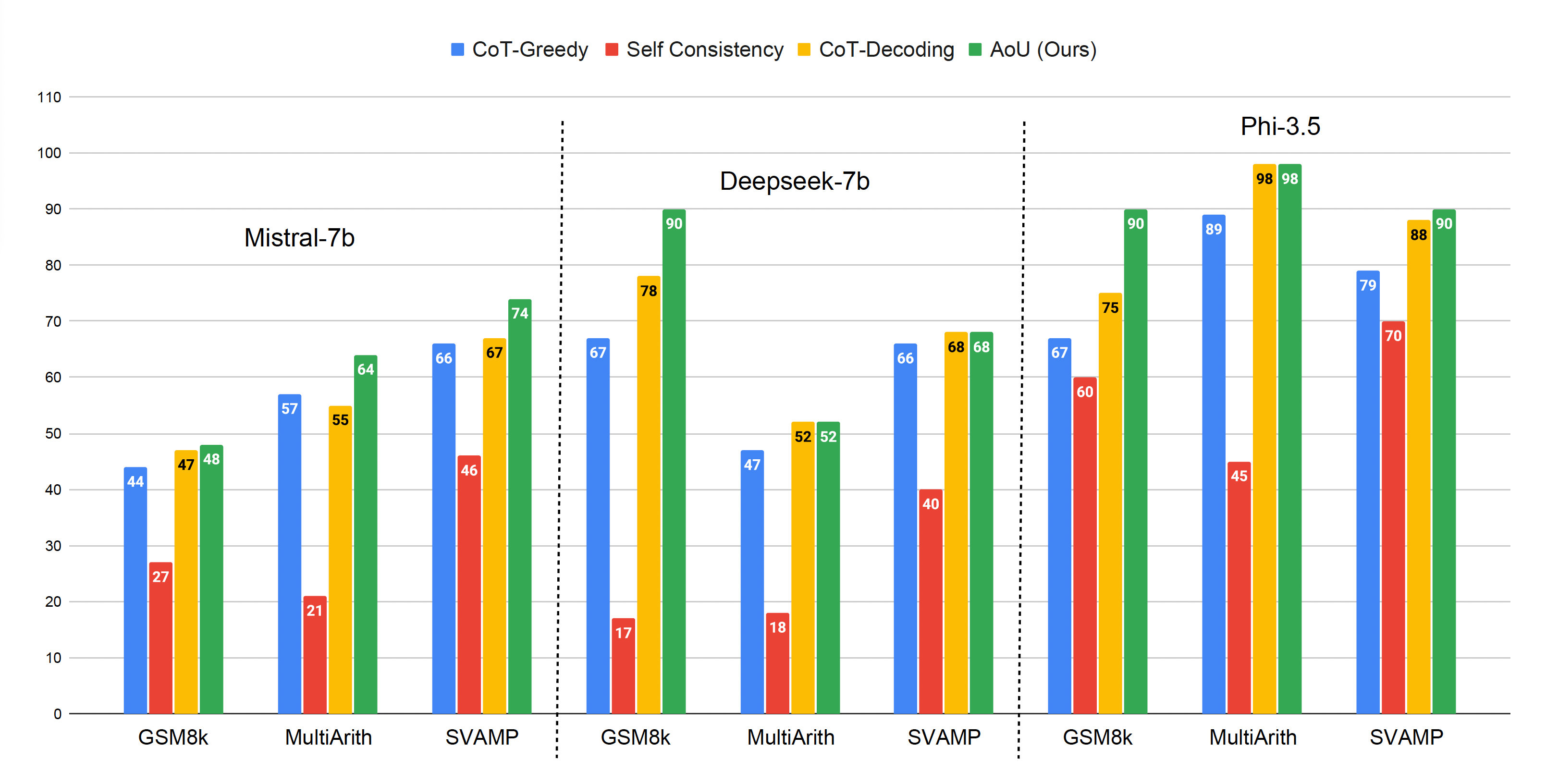} 
\caption{
Performance comparison of different models (\textsc{Mistral-7B-v0.3}, \textsc{DeepSeek-7B}, and \textsc{Phi-3.5}) across three datasets (GSM8k, MultiArith, SVAMP) under various decoding strategies: CoT-Greedy (blue), Self-Consistency (red), CoT-Decoding (yellow), and our proposed AoU (green).}
\label{fig:results}
\end{figure*}

\paragraph{Models.} 
Three publicly available LLMs were used to evaluate AoU against several baselines. (1) \textsc{Mistral-7B}~\citep{jiang2023mistral7b}\footnote{\url{https://huggingface.co/mistralai/Mistral-7B-Instruct-v0.3}} is a general-purpose decoder model optimized for scalable performance, (2) \textsc{DeepSeek-7B}~\citep{deepseekai2025deepseekr1incentivizingreasoningcapability}\footnote{\url{https://huggingface.co/deepseek-ai/deepseek-llm-7b-chat}} is instruction-tuned for multi-turn dialogue and enhanced reasoning alignment, and (3) \textsc{Phi-3.5 Mini}~\citep{abdin2024phi3technicalreporthighly}\footnote{\url{https://huggingface.co/microsoft/Phi-3.5-mini-instruct}} is a compact, cost-efficient model designed for educational and low-resource reasoning scenarios. This model selection enables evaluation across different reasoning behaviors and architectural priors.

\paragraph{Baselines.} 
We compare our method against three strong prompting-based baselines. \textsc{Chain-of-Thought (CoT)}~\citep{cot_wei} guides the model to articulate intermediate reasoning steps before producing a final answer. \textsc{Self-Consistency}~\citep{wang2023selfconsistency} increases samples \( n = 20 \) reasoning paths and returns the most frequent outcome. \textsc{CoT-Decoding}~\citep{wang2024chainofthought} removes explicit reasoning instructions, instead relying on various decoding paths to elicit inherent reasoning strategies in LLMs. These baselines represent standard approaches for encouraging multi-step reasoning without assumption-level control.

\paragraph{Experimental Setup.} 
All models are evaluated as released, without any additional fine-tuning or instruction-tuning. Generation is carried out using a fixed temperature of 1.0 and a maximum output length of 526 tokens. Each input is processed with a single decoding pass. The experiments were carried out on a high performance set-up equipped with a NVIDIA
A100 GPU. To ensure reproducibility, random seeds are fixed and all decoding parameters remain constant across experiments.

\section{Results}
\label{sec:results}

Experimental results are illustrated in Fig.~\ref{fig:results}. AoU prompting achieves state-of-the-art performance across a range of LLMs and math reasoning tasks without relying on external tools. Its ability to constrain reasoning to validated assumptions leads to higher accuracy, improved robustness, and reduced variability compared to standard chain-of-thought strategies. Note that our guarantees concern trace-level faithfulness (no unsupported premises enter the reasoning trace); correctness still depends on the choice of $P(\boldsymbol g\mid Q,\mathcal G^{+})$ used to marginalize unvalidated premises.

\paragraph{Overall Performance.}
Across all datasets and models, AoU consistently outperforms or matches the strongest baseline. On \textsc{Mistral-7B}, AoU achieves 48\%, 64\%, and 74\% accuracy on GSM8k, MultiArith, and SVAMP respectively, improving over CoT-Greedy by +4–8\% and outperforming Self-Consistency by wide margins (e.g., +45\% on MultiArith). Similar trends hold for \textsc{DeepSeek-7B}, where AoU reaches 90\% on GSM8k, 52\% on MultiArith, and 68\% on SVAMP, matching or exceeding the strongest competing reasoning baseline on each task. On the more capable \textsc{Phi-3.5}, the gains narrow as base model performance saturates, but AoU still yields marginal improvements and demonstrates consistent behavior across all tasks.

\paragraph{Comparisons with Baselines.}
Self-Consistency~\citep{wang2023selfconsistency} performs inconsistently across models, often underperforming CoT-Greedy~\citep{cot_wei}, particularly on GSM8k and MultiArith. This suggests that sampling-based diversity does not reliably enhance reasoning faithfulness without structured constraint. CoT-Decoding~\citep{wang2024chainofthought}, which removes explicit prompting and relies on latent reasoning patterns, offers noticeable improvements over greedy decoding in some cases, but remains less stable than AoU. In contrast, AoU delivers robust improvements across models and datasets, highlighting its effectiveness in controlling hallucinated reasoning.

\paragraph{Effect of Model Type.}
As expected, more reasoning-tuned models such as \textsc{Phi-3.5} show smaller absolute gains from AoU, though improvements remain consistent. This suggests that AoU not only helps weaker models reason more faithfully, but also encourages even stronger models to avoid inaccurate inferences and hallucinations, particularly on under-determined or adversarially structured problems.

 \section{Limitations}
 \label{sec:limitations}


While AoU prompting improves reasoning faithfulness across tasks and models, it has some limitations.

\paragraph{Model Dependence.}
AoU assumes the model can reliably judge which assumptions are supported or missing—a form of meta-reasoning that may not generalize across weaker or poorly aligned models. In such cases, the audit step may be overly strict or lenient.

\paragraph{No External Verification.}
AoU operates without retrieval or external tools, improving simplicity and interpretability. However, this limits its ability to fact-check assumptions that are not explicitly stated but necessary for correct answers.

\section{Conclusion and Future Work}
\label{sec:conclusion}
This work introduces AoU prompting, a structured framework for reducing reasoning-induced hallucinations in LLMs. By explicitly separating assumption identification, validation, and constrained reasoning, AoU improves accuracy, interpretability, and faithfulness across diverse mathematical benchmarks. Our method operates without access to external tools or post-hoc verification, offering a lightweight and generalizable approach to controlled generation. While AoU shows strong performance gains, it also opens several avenues for future research. These include extending the framework to non-mathematical domains such as commonsense or scientific reasoning, integrating uncertainty-aware reasoning over weakly supported assumptions, and scaling to longer, multi-turn dialogues. Further, combining AoU with retrieval or formal verification could enhance both factual grounding and logical consistency. 

\bibliographystyle{apalike}

\end{document}